\documentclass{article}
\usepackage{ijcai19}

\usepackage{times}
\usepackage{soul}
\usepackage{url}
\usepackage[hidelinks]{hyperref}
\usepackage[utf8]{inputenc}
\usepackage[small]{caption}
\usepackage{graphicx}
\usepackage{amsmath}
\usepackage{booktabs}
\usepackage{algorithm}
\usepackage{algorithmic}
\usepackage{subfigure}
\usepackage{amsfonts}
\usepackage{amsthm}
\usepackage{mathtools}
\usepackage{color}
\usepackage{comment}

\usepackage{eqparbox}

\title{Reparameterizable Subset Sampling via Continuous Relaxations}

\author{
Sang Michael Xie
\textnormal{and}
Stefano Ermon\\
\affiliations
Stanford University\\
\emails
\{xie, ermon\}@cs.stanford.edu
}

\newtheorem{prop}{Proposition}
\newtheorem{theorem}{Theorem}
\begin{document}

\maketitle
\frenchspacing
\begin{abstract}
Many machine learning tasks require sampling a subset of items from a collection based on a parameterized distribution.  The Gumbel-softmax trick can be used to sample a single item, and allows for low-variance reparameterized gradients with respect to the parameters of the underlying distribution. However, stochastic optimization involving subset sampling is typically not reparameterizable. To overcome this limitation, we define a continuous relaxation of subset sampling that provides reparameterization gradients by generalizing the Gumbel-max trick. We use this approach to sample subsets of features in an instance-wise feature selection task for model interpretability, subsets of neighbors to implement a deep stochastic k-nearest neighbors model, and sub-sequences of neighbors to implement parametric t-SNE by directly comparing the identities of local neighbors. We improve performance in all these tasks by incorporating subset sampling in end-to-end training.
\end{abstract}

\section{Introduction}
Sampling a single item from a collection is common in machine learning problems such as generative modeling with latent categorical variables, attention mechanisms~\cite{KingmaSSL,XuShowAttend}.
These tasks involve optimizing an expectation objective over a latent categorical distribution parameterized by a neural network.
Score-based methods such as REINFORCE \cite{Williamsreinforce} for estimating the gradient of such objectives typically have high variance.
The reparameterization trick \cite{kingma2013autoencoding} allows for low variance gradients for certain distributions, not typically including categorical distributions.
The Gumbel-softmax trick \cite{janggumbelsoftmax} or Concrete distribution \cite{maddison2017concrete} are continuous relaxations that allow for reparameterized gradients with respect to the parameters of the distribution. 
Among many others, this enabled generative modeling with latent categorical variables without costly marginalization and modeling sequences of discrete elements with GANs \cite{janggumbelsoftmax,KusnerDiscreteGAN}.

In this paper, we consider the more general problem of sampling a subset of \emph{multiple items} from a collection without replacement.
As an example, choosing a subset is important in instance-wise feature selection \cite{chen18L2X}, where the goal is to select a subset of features that best explain the model's output for each example.
Sampling subsets of neighbors also enables implementing stochastic $k$-nearest neighbors end-to-end with deep features.
Stochastic optimization involving subset sampling, however, does not typically have relaxations with low-variance reparameterization gradients as in Gumbel-softmax.
To overcome this limitation, we develop a continuous relaxation for approximate reparameterized gradients with respect to the parameters of a subset distribution to enable learning with backpropagation.
In our setting, the Gumbel-max trick (and thus Gumbel-softmax) is not directly applicable since it requires treating every possible subset as a category, requiring a combinatorial number of categories.
We use an extension to the Gumbel-max trick which perturbs the log-probabilities of a categorical distribution with Gumbel noise and takes the top-$k$ elements to produce samples without replacement. Ignoring ordering in these samples allows for sampling from a family of subset distributions using the same algorithm. We give a general algorithm that produces continuous relaxations with reparameterization gradients using top-$k$ relaxations. We then show that a recent top-$k$ relaxation \cite{NN3} can be used in our algorithm and study the consistency of this top-$k$ relaxation.

Our main contributions are the following:
\begin{itemize}
  \item We give an 
  algorithm for a reparameterizable continuous relaxation to sampling subsets
  using top-$k$ relaxations and a extension to the Gumbel-max trick.
  \item We show that the top-$k$ relaxation of \cite{NN3} is \textit{consistent} in the sense that the ordering of inputs is preserved in the output in many practical settings.
  \item We test our algorithm as a drop-in replacement for subset selection routines in explaining deep models through feature subset selection, training stochastic neural k-nearest neighbors, and implementing parametric t-SNE without Student-t distributions by directly comparing neighbor samples. We improve performance on all tasks using the same architectures and metrics as the original~\footnote{Code available at \url{https://github.com/ermongroup/subsets}.}.
\end{itemize}

\section{Preliminaries}

\subsection{Weighted Reservoir Sampling}
\label{wrssection}
Reservoir sampling is a family of streaming algorithms that is used to sample $k$ items from a collection of $n$ items, $x_1,\dots,x_n$, where $n$ may be infinite \cite{Vitterreservoir}. We consider finite $n$ and only produce samples after processing the entire stream. In weighted reservoir sampling, every $x_i$ is associated with a weight $w_i\geq 0$. Let $\mathbf{w}=[w_1,\dots,w_n]$ and $Z=\sum_{i=1}^n w_i$ be the normalizing constant. 
Let $\mathbf{e}^j=[e^j_1,\dots,e^j_n] = [ 0, \cdots, 0, 1, 0, \cdots, 0]\in \{0,1\}^n$ be a 1-hot vector, i.e., a vector with only one nonzero element at index $j$, where $e^j_j = 1$. We define a weighted reservoir sample (WRS) as $S_{wrs}=[\mathbf{e}^{i_1},\dots, \mathbf{e}^{i_k}]$, a sequence of $k$ 1-hot (standard basis) vectors where $\mathbf{e}^{i_j}$ represents selecting element $x_{i_j}$ in the $j$-th sample. 
We wish to sample $S_{wrs}$ from
\begin{equation}
\label{subsetdist}
    p(S_{wrs}\mid \mathbf{w})= \frac{w_{i_1}}{Z} \frac{w_{i_2}}{Z-w_{i_1}}\cdots \frac{w_{i_k}}{Z-\sum_{j=1}^{k-1} w_{i_j}},
\end{equation}
which corresponds to sampling without replacement with probabilities proportional to item weights. Modeling samples without replacement allows for sampling a sequence of distinct items. For $k=1$, $p(S_{wrs}|\mathbf{w})$ is the standard softmax distribution with logits given by $\log(w_i)$.

\begin{algorithm}[tb]
\caption{Weighted Reservoir Sampling (non-streaming)}
\label{weightedreservoirsample}
\begin{algorithmic}[1]
\REQUIRE Items $x_1, \dots, x_n$, weights $\mathbf{w}=[w_1,\dots,w_n]$, reservoir size $k$
\ENSURE $S_{wrs}=[\mathbf{e}^{i_1},\dots, \mathbf{e}^{i_k}]$ a sample from $p(S_{wrs}|\mathbf{w})$
\STATE $\mathbf{r}\gets$ [ ]
\FOR  {$i\gets1$ to $n$}
    \STATE $u_i\gets$ Uniform(0, 1)
    \STATE $r_i\gets u_i^{1/w_i}$ \COMMENT{Sample random keys}
    \STATE $\mathbf{r}$.append($r_i$)
\ENDFOR
\STATE $[\mathbf{e}^{i_1},\dots, \mathbf{e}^{i_k}]\gets$ TopK($\mathbf{r}$, $k$)
\RETURN $[\mathbf{e}^{i_1},\dots, \mathbf{e}^{i_k}]$
\end{algorithmic}
\label{wrs}
\end{algorithm}

\cite{EfraimidisReservoir} give an algorithm for weighted reservoir sampling (Algorithm \ref{wrs}). Each item $x_i$ is given a random key $r_i=u_i^{1/w_i}$ where $u_i$ is drawn from a uniform distribution between $[0, 1]$ and $w_i$ is the weight of item $x_i$. Let the top $k$ keys over the $n$ items be $r_{i_1}, \dots, r_{i_k}$.
We define the function $\text{TopK}(\mathbf{r},k)$ which takes keys $\mathbf{r}=[r_1,\dots,r_n]$ and returns $[\mathbf{e}^{i_1},\dots, \mathbf{e}^{i_k}]$ associated with the top-$k$ keys.
The algorithm uses $\text{TopK}$ to return the items $S_{wrs}=[\mathbf{e}^{i_1},\dots, \mathbf{e}^{i_k}]$ as the WRS. 
\citeauthor{EfraimidisReservoir} proved (Proposition 5 in \cite{EfraimidisReservoir}) that the output of Algorithm \ref{wrs} is distributed according to $p(S_{wrs}|\mathbf{w})$.

\subsection{Gumbel-max Trick}
Given $\mathbf{w}$ as in (\ref{subsetdist}), $\log(w_i)$ are logits for a softmax distribution $p(x_i|\mathbf{w})=w_i / Z$. The Gumbel-max trick \cite{YELLOTT1977109} generates random keys $\hat{r}_i=\log(w_i)+g_i$ by perturbing logits with Gumbel noise $g_i \sim \text{Gumbel}(0,1)$, then taking $x_{i^*}$ such that $i^*=\arg\max_{i}\hat{r}_i$ as a sample. These samples are distributed according to $p(x_i|\mathbf{w})=w_i / Z$.
The idea is to reparameterize the sample as a deterministic transformation of the parameters $\mathbf{w}$ and some independent noise $g_i$. Then by relaxing the deterministic transformation (from \emph{max} to \emph{softmax}), the Gumbel-softmax trick allows for training with backpropagation \cite{maddison2017concrete,janggumbelsoftmax}.
Similarly, we use an extension of the Gumbel-max trick to decouple the deterministic transformation of the parameters (in our case, a top-$k$ selection function) and the randomness (Gumbel noise $g_i$), and we relax the top-$k$ function to allow for backpropagation.

\section{Reparameterizable Continuous Relaxation for Subset Sampling}

\subsection{Setup}
We represent a subset $S\in \{0,1\}^n$ as a $k$-hot vector, which is a vector with exactly $k$ nonzero elements that are all equal to 1. We define the probability of a subset $S$ as the sum of the probabilities of all WRS with the same elements  
\begin{equation}
    \label{subsetdist2}
    p(S \mid \mathbf{w}) = \sum_{S_{wrs}\in \Pi(S)}p(S_{wrs}|\mathbf{w})
\end{equation}
where $\Pi(S)=\{S_{wrs}:S=\sum_{j=1}^k S_{wrs}[j]\}$ is the set of all permutations of elements in $S$ represented by sequences of 1-hot vectors. Here, $S_{wrs}[j]$ is the $j$-th 1-hot vector in the sequence.
By simply ignoring the order of elements in a WRS, we can also sample from $p(S|\mathbf{w})$ using the same algorithm. 
Note that this is a restricted family of subset distributions.
Since each distribution is over $\binom{n}{k}$ subsets of size $k$, the full space of distributions requires $\binom{n}{k}-1$ free parameters. Here, we reduce the number of free parameters to $n-1$.
While this is a restriction, we gain tractability in our algorithm.

\subsection{Gumbel-max Extension}
We extend the Gumbel-max trick to sample from $p(S|\mathbf{w})$. The main intuition is that the outputs of the reservoir sampling Algorithm \ref{wrs} only depend on the \emph{ordering} of the random keys $r_i$ and not their values.
We show that random keys $\hat{r}_i$ generated from the Gumbel-max trick are monotonic transformations of the random keys $r_i$ from Algorithm \ref{wrs}. 
Because a monotonic transformation preserves the ordering, the elements that achieve the top-$k$ Gumbel-max keys $\hat{r}_i$ have the same distribution as the elements that achieve the top-$k$ weighted reservoir sampling keys $r_i$. 
Therefore we can sample from $p(S|\mathbf{w})$ by taking the top-$k$ elements of $\hat{r}_i$ instead. 

To make the procedure differentiable with respect to $\mathbf{w}$, we replace top-$k$ selection with a differentiable approximation. 
We define a relaxed $k$-hot vector $\mathbf{a}=[a_1,\dots,a_n]$ to have $\sum_{i=1}^n a_i = k$ and $0\leq a_i \leq 1$. We relax $S_{wrs}$ by replacing all $\mathbf{e}^{i_j}$ with relaxed 1-hot vectors and relax $S$ by a relaxed $k$-hot vector. Our continuous relaxation will approximate sampling from $p(S|\mathbf{w})$ by returning relaxed $k$-hot vectors.
We use a top-$k$ relaxation RelaxedTopK, which is a \emph{differentiable} function that takes $\hat{r}_i$, $k$, and a temperature parameter $t>0$ and returns a relaxed $k$-hot vector $\mathbf{a}$ such that as $t\rightarrow 0$, $\text{RelaxedTopK}(\mathbf{\hat{r}},k,t)\rightarrow \sum_{j=1}^k \text{TopK}(\mathbf{\hat{r}},k)[j]$, 
where $\text{TopK}(\mathbf{\hat{r}},k)[j]=\mathbf{e}^{i_j}$ is the 1-hot vector associated with the $j$-th top key in $\mathbf{\hat{r}}$.
Thus Algorithm \ref{diffsubsample} produces relaxed $k$-hot samples that, as $t\rightarrow 0$ converge to exact samples from $p(S|\mathbf{w})$. 
Note that we can also produce approximate samples from $p(S_{wrs}|\mathbf{w})$ if an intermediate output of RelaxedTopK is a sequence of $k$ relaxed 1-hot vectors $[\mathbf{a}^{i_1},\dots,\mathbf{a}^{i_k}]$ such that as $t\rightarrow 0$, $[\mathbf{a}^{i_1},\dots,\mathbf{a}^{i_k}]\rightarrow \text{TopK}(\mathbf{\hat{r}},k)$.
This means that the intermediate output converges to a WRS.
\begin{algorithm}[tb]
\caption{Relaxed Subset Sampling}
\label{diffsubsample}
\begin{algorithmic}[1]
\REQUIRE Items $x_1, \dots, x_n$, weights $\mathbf{w}=[w_1,\dots,w_n]$, subset size $k$, temperature $t>0$
\ENSURE Relaxed $k$-hot vector $\mathbf{a}=[a_1,\dots, a_n]$, where $\sum_{i=1}^n a_i = k$, $0\leq a_i\leq 1$

\STATE $\mathbf{\hat{r}}\gets$ [ ]
\FOR  {$i\gets1$ to $n$}
    \STATE $u_i\gets$ Uniform(0, 1) \COMMENT{Random Gumbel keys}
    \STATE $\hat{r}_i\gets -\log(-\log(u_i)) + \log(w_i)$
    \STATE $\mathbf{\hat{r}}$.append($\hat{r}_i$)
\ENDFOR
\STATE $\mathbf{a}\gets$ RelaxedTopK($\mathbf{\hat{r}}$, $k$, $t$)\\
\RETURN $\mathbf{a}$
\end{algorithmic}
\end{algorithm}

\begin{prop}
\label{exactsamples}
Let RelaxedTopK be defined as above. Given $n$ items $x_1,\dots,x_n$, a subset size $k$, and a distribution over subsets described by weights $w_1,\dots,w_n$, Algorithm \ref{diffsubsample} gives exact samples from $p(S|\mathbf{w})$ as in (\ref{subsetdist2}) as $t\rightarrow 0$.
\end{prop}
\begin{proof}
Let the random keys in Algorithm \ref{diffsubsample} be $\mathbf{\hat{r}}$ and the random keys in weighted reservoir sampling be $\mathbf{r}$. For any $i$,
\begin{align*}
    \hat{r}_i &= -\log(-\log(u_i)) + \log(w_i)\\
    &=-\log(-\frac{1}{w_i}\log(u_i))\\
    &=-\log(-\log(u_i^{1/w_i}))=-\log(-\log(r_i)).
\end{align*}
Fixing $u_i$, since $-\log(-\log(a))$ is monotonic in $a$, $\hat{r}_i$ is a monotonic transformation of $r_i$ and $\text{TopK}(\mathbf{\hat{r}},k)=\text{TopK}(\mathbf{r},k)$. 
Let $\text{TopK}(\mathbf{r},k)$ be samples from Algorithm \ref{weightedreservoirsample}. By construction of (\ref{subsetdist2}), $\sum_{j=1}^k\text{TopK}(\mathbf{r},k)[j]$ is distributed as $p(S|\mathbf{w})$. As $t\rightarrow 0$, Algorithm \ref{diffsubsample} produces samples from $p(S|\mathbf{w})$ since $\text{RelaxedTopK}(\mathbf{\hat{r}},k,t)\rightarrow \sum_{j=1}^k\text{TopK}(\mathbf{\hat{r}},k)[j]=\sum_{j=1}^k\text{TopK}(\mathbf{r},k)[j]$.
\end{proof}
This fact has been shown previously in \cite{vieirablog} for $k=1$ and in \cite{KimILP} for sampling from $p(S_{wrs}|\mathbf{w})$ without the connection to reservoir sampling. \citeauthor{kool19astochasticbeams} concurrently developed a similar method for ordered sampling without replacement for stochastic beam search.
Note that Algorithm \ref{diffsubsample} is general to any choice of top-$k$ relaxation. 

\subsection{Differentiable Top-k Procedures}
A vital component of Algorithm \ref{diffsubsample} is a top-$k$ relaxation that is \emph{differentiable} with respect to the input keys $\hat{r}_i$ (a random function of $\mathbf{w}$). This allows for parameterizing $\mathbf{w}$, which governs $p(S|\mathbf{w})$, using neural networks and training using backpropagation.
We propose to use a recent top-$k$ relaxation based on successive applications of the softmax function \cite{NN3}. For some temperature $t>0$, define for all $i=1,\dots,n$
\begin{align}
    \alpha_i^1 \coloneqq \hat{r}_i, \enspace \alpha_i^{j+1} \coloneqq \alpha_i^j + \log(1-a_i^j)
\end{align}
where $a_i^j$ is a sample at step $j$ from the distribution
\begin{align}
    p(a_i^j = 1)=\frac{\exp(\alpha_i^j/t)}{\sum_{m=1}^n \exp(\alpha_i^j/t)}
\end{align}
for $j=1,\dots,k$ steps.
In the relaxation, the $a_i^j$ is replaced with its expectation, $p(a_i^j = 1)$, such that the new update is 
\begin{align}
    \alpha_i^{j+1} \coloneqq \alpha_i^j + \log(1-p(a_i^j = 1))
\end{align}
and the output is $a_i=\sum_{j=1}^k p(a_i^j = 1)$ for each $i$.
Let $\mathbf{a}^j=[p(a_1^j=1),\dots,p(a_n^j=1)]$ be relaxed 1-hot outputs at step $j$ and $\mathbf{a}=\sum_{j=1}^k\mathbf{a}^j$ be the relaxed $k$-hot output.
\citeauthor{NN3} show that as $t\rightarrow 0$, $[\mathbf{a}^1,\dots,\mathbf{a}^k]\rightarrow \text{TopK}(\mathbf{\hat{r}},k)$ and thus $\mathbf{a}\rightarrow \sum_{j=1}^k\text{TopK}(\mathbf{\hat{r}},k)[j]$ so that this is a valid RelaxedTopK. Thus the relaxation can be used for approximately sampling from both $p(S|\mathbf{w})$ and $p(S_{wrs}|\mathbf{w})$.
Next we show that the magnitude of values in the output relaxed $k$-hot vector $\mathbf{a}$ preserves order of input keys $\hat{r}_i$ for $t\geq 1$.
\begin{theorem}
\label{consistency}
Given keys $\mathbf{\hat{r}}$, the top-$k$ relaxation of \cite{NN3} produces a relaxed $k$-hot vector $\mathbf{a}$ where if $\hat{r}_i\leq \hat{r}_j$, then $a_i\leq a_j$ for any temperature $t\geq1$ and $k\leq n$.
\end{theorem}
\begin{proof}
By induction on $k$. Fix any $\hat{r}_i\leq \hat{r}_j$. For step $k=1$, we have $a_i^1\leq a_j^1$ since the softmax function preserves ordering. Assuming the statement holds for $0,\dots,k$, we want to show that $\alpha_i^{k+1}\leq \alpha_j^{k+1}$, which suffices to imply $a_i^{k+1}\leq a_j^{k+1}$ by the order-preserving property of softmax.
Define $\hat{\alpha}_i^k=\frac{\exp(\alpha_i^k / t)}{\sum_{m=1}^n \exp(\alpha_i^k / t)}=p(a_i^j=1)$. Then
\begin{align*}
    \alpha_i^{k+1}&=\alpha_i^k+\log\left(1-\hat{\alpha}_i^k\right)\\
    \exp(\alpha_i^{k+1})&=\exp(\alpha_i^k)\left(1-\hat{\alpha}_i^k\right).
\end{align*}
Comparing $\alpha_j^{k+1}$ and $\alpha_i^{k+1}$ through the ratio,
\begin{align*}
    \frac{\exp(\alpha_j^{k+1})}{\exp(\alpha_i^{k+1})} = \frac{\exp(\alpha_j^k)}{\exp(\alpha_i^k)}\left(\frac{\exp(\alpha_i^k/t)+c}{\exp(\alpha_j^k/t)+c}\right)
\end{align*}
where we can view $c=\sum_{m=1}^n\exp(\alpha_m^k/t)-\exp(\alpha_i^k/t)-\exp(\alpha_j^k/t)\geq 0$ as a non-negative constant in this analysis. Note that $\frac{\exp(\alpha_j^k)}{\exp(\alpha_i^k)}=\exp(\alpha_j^k-\alpha_i^k)\geq \exp((\alpha_j^k-\alpha_i^k)/t)=\frac{\exp(\alpha_j^k/t)}{\exp(\alpha_i^k/t)}$ for $t\geq 1$. Therefore
\begin{align*}
    \frac{\exp(\alpha_j^{k+1})}{\exp(\alpha_i^{k+1})} \geq \frac{\exp(\alpha_j^k/t)}{\exp(\alpha_i^k/t)}\left(\frac{\exp(\alpha_i^k/t)+c}{\exp(\alpha_j^k/t)+c}\right)\geq 1
\end{align*}
for any $c\geq 0$. Thus $\alpha_i^{k+1}\leq \alpha_j^{k+1}$, implying $a_i^{k+1}\leq a_j^{k+1}$.
\end{proof}
Therefore, this top-$k$ relaxation is \textit{consistent} in the sense that the indices with the top $k$ weights in $w$ are also the top $k$ values in the relaxed $k$-hot vector for many reasonable $t$. In the limit as $t\rightarrow 0$ the consistency property holds for the exact $k$-hot vector, but the relaxation is not consistent in general for $0<t<1$. To see a concrete example of the loss of consistency at lower $t$, take $\mathbf{\hat{r}}=[e, e^2]$, $k=2$, and $t=0.4$. The output on this input is $[1.05, 0.95]$, which has inconsistent ordering. Consistency at higher temperatures is important when considering the bias-variance tradeoff in the gradients produced by the approximation. While higher temperatures allow for lower variance gradients, the bias introduced may cause the model to optimize something far from the true objective. Consistency in ordering suggests that the top-$k$ information is not lost at higher temperatures.

Note that any top-$k$ relaxation can be used in Algorithm \ref{diffsubsample}. For example, \cite{grover2018stochastic} give a relaxation of the sorting procedure through a relaxed permutation matrix $P$. To produce a relaxed $k$-hot vector in this framework, we can take $\mathbf{a}=\sum_{j=1}^k P_j$ where $P_j$ is the $j$-th row of $P$. Similarly, when the associated temperature parameter approaches 0, $P$ approaches the exact permutation matrix so that the relaxed $k$-hot vector $\mathbf{a}$ also approaches the exact top-$k$ $k$-hot vector.

\section{Experiments}
\begin{figure}
  \centering
    \includegraphics[width=0.382\textwidth]{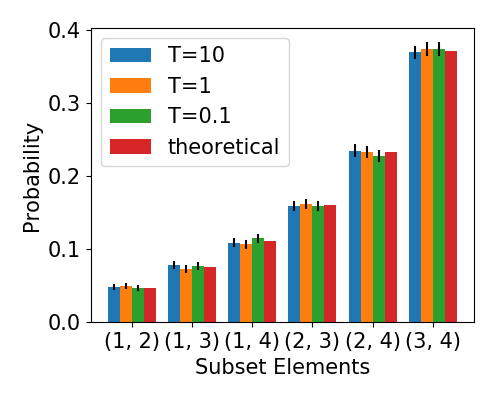}
\caption{Empirical subset distributions on a toy subset distribution.}
\end{figure}
\subsection{Synthetic Experiments}
We check that generating samples from Algorithm \ref{diffsubsample} results in samples approximately from $p(S | \mathbf{w})$. We define a subset distribution using weights $\mathbf{w}=[0.1, 0.2, 0.3, 0.4]$ and take subset size $k=2$. Using Algorithm \ref{diffsubsample} and the top-$k$ relaxation from \cite{NN3}, we sample 10000 relaxed $k$-hot samples for each temperature in $\{0.1, 1, 10\}$ and take the the top-$k$ values in the relaxed $k$-hot vector as the ``chosen'' subset. 
We plot the empirical histogram of subset occurrences and compare with the true probabilities from the subset distribution, with 95\% confidence intervals. The relaxation produces subset samples with empirical distribution within 0.016 in total variation distance of $p(S|\mathbf{w})$ for all $t$.
This agrees with Theorem~\ref{consistency}, which states that even for higher temperatures, taking the top-$k$ values in the relaxed $k$-hot vector should produce true samples from (\ref{subsetdist2}).

\subsection{Model Explanations}
\begin{figure}
\centering
\noindent\fbox{%
    \parbox{0.468\textwidth}{%
    \footnotesize
i just \hl{saw} this at the - film \hl{festival} it was the most \hl{wonderful} movie the \hl{best} i've seen in quite a while the - character of is an open - and as - in the film \hl{love} filled person portrayed by the beautiful - the - music is - and - a - to the soul each character \hl{seems} to be - in some way and their - w conflicts make for a \hl{great} story the tale told by a k - as the old man was most - done i wanted it to go on and on that - \hl{seemed} to remember his place throughout added get power to the story a \hl{refreshing} change to the - headed plots of many modern writers all and all an \hl{excellent} film go see
\normalsize
    }%
}
\caption{An example of an IMDB review where our explainer model predicts the $k=10$ words and the sentiment model outputs a similar prediction (positive) when only taking the 10 words as input. Words not in the vocabulary are replaced with dashes.}
\label{reviewexample}
\end{figure}
\begin{figure}
\centering
\noindent\fbox{%
    \parbox{0.46\textwidth}{%
    \footnotesize
\hl{read} the \hl{book} forget the \hl{movie}
\normalsize
    }%
}
\caption{A review where the L2X explainer model \protect\cite{chen18L2X} can select 10 words (whole review) but subsamples, reducing the sentiment model's confidence from 0.93 to 0.52 when taking the subset of words as input. Our explainer returns the whole review.}
\label{l2xsubsample}
\end{figure}
We follow the L2X model \cite{chen18L2X} and set up the problem of explaining instance-wise model outputs by training an auxiliary \textit{explainer} model to output the $k$ features with the highest mutual information with the model output.
For example, given a movie review and a sentiment model, we want to select up to $k$ words from the review that best explains the sentiment model's output on the particular review (see Figure \ref{reviewexample}).
Since optimizing the mutual information is intractable,
L2X optimizes a variational lower bound instead:
\begin{align}
    \max_{\mathcal{E},q} \mathbb{E}[\log q(X_S)] \text{ \enspace s.t. } S\sim \mathcal{E}(X)
\end{align}
where $\mathcal{E}:\mathbb{R}^d\rightarrow \mathcal{P}_k$ is an explainer model, parameterized as a neural network, mapping from an input to the space of all $k$-hot vectors $S$.
The approximating distribution is $q(X_S)$ where $q$ is a neural network and $X_S=S\odot X \in \mathbb{R}^d$ is $X$ with the elements not corresponding to $S$ zeroed out. 
L2X approximates the subset sampling procedure by sampling $k$ independent times from a Concrete distribution \cite{maddison2017concrete} and taking the elementwise maximum. Since each independent Concrete sample is a relaxed 1-hot vector, the L2X explainer model may suffer by independently sampling the same feature many times, resulting in selecting less than $k$ features. Our model differs by replacing this sampling procedure with Algorithm \ref{diffsubsample}. Figure \ref{l2xsubsample} shows an instance where the L2X explainer model chooses an ambiguous subset of the review even when the review is shorter than the maximum number of words to be selected ($k$), which reduces the sentiment model's confidence significantly 
(0.93 to 0.52). Our model selects the entire review in this case, so the sentiment model's output is not affected.

We test our results on the Large Movie Review Dataset (IMDB) for sentiment classification \cite{maasIMDB}, where we select the most important words or sentences that contribute to the sentiment prediction for the review. The original model for word-based sentiment classification (IMDB-word) is a convolutional neural network \cite{kim2014convolutional}, while the original model for sentence-based sentiment prediction is a hierarchical LSTM (IMDB-sent) \cite{LiHierarchicalLSTM}. The explainer and variational distribution models are CNNs with the same architectures as in L2X \cite{chen18L2X}. 
Following L2X, we use $k=10$ for IMDB-word and $k=1$ sentences for IMDB-sent. At test time, all explainer models deterministically choose subsets based on the highest weights instead of sampling. We evaluate using \textit{post-hoc accuracy}, which is the proportion of examples where the original model evaluated on masked features $X_S$ matches the model on unmasked $X$. We use cross validation to choose temperatures $t\in \{0.1, 0.5, 1, 2, 5\}$ according to the validation loss. Our model (RelaxSubSample) improves upon L2X by up to 1\% by only changing the sampling procedure (Table~\ref{explanation}).

\begin{table}
\centering
\resizebox{0.32\textwidth}{!} {
\begin{tabular}{lrr}  
\toprule
Model    & IMDB-word & IMDB-sent \\
\midrule
L2X   & $90.7\pm 0.004$ & $82.9 \pm 0.005$    \\
RelaxSubSample    & $\mathbf{91.7\pm 0.003}$ & $\mathbf{83.2 \pm 0.004}$    \\
\bottomrule
\end{tabular}
}
\caption{Post-hoc accuracy (\%, 95\% interval) on explaining sentiment predictions on the IMDB Large Movie Review Dataset. L2X refers to the model from \protect\cite{chen18L2X} while RelaxSubSample is our method.}
\label{explanation}
\end{table}

\subsection{Stochastic K-Nearest Neighbors}

\begin{table}
\centering
\resizebox{0.39\textwidth}{!} {
\begin{tabular}{lrrr}  
\toprule
Model & MNIST & Fashion-MNIST & CIFAR-10 \\
\midrule
Stochastic NeuralSort  & \textbf{99.4} & 93.4 & 89.5   \\
RelaxSubSample & 99.3 & \textbf{93.6} & \textbf{90.1} \\
CNN (no kNN) & 99.4 & 93.4 & 95.1\\
\bottomrule
\end{tabular}
}
\caption{Classification test accuracy (\%) of deep stochastic k-nearest neighbors using the NeuralSort relaxation \protect\cite{grover2018stochastic} and our method (RelaxSubSample).}
\label{fig:knntable}
\end{table}

\begin{table}
\centering
\resizebox{0.34\textwidth}{!} {
\begin{tabular}{lrrr}  
\toprule
Model    & $m=100$ & $m=1000$ & $m=5000$ \\
\midrule
NeuralSort   & \textbf{0.010s} & 0.073s & 3.694s \\
RelaxSubSample  & \textbf{0.010s} & \textbf{0.028s} & \textbf{0.110s}\\
\bottomrule
\end{tabular}
}
\caption{Forward pass average runtimes (100 trials) for a small CNN selecting $k=5$ neighbors from $m$ candidates using the NeuralSort relaxation \protect\cite{grover2018stochastic} and our method (RelaxSubSample). Results were obtained from a Titan Xp GPU.}
\label{fig:knntimings}
\end{table}

We give a stochastic k-NN algorithm where we use deep features tuned end-to-end for computing neighbors. 
We follow the setup of \citeauthor{grover2018stochastic}, using $m=100$ randomly sampled neighbor candidates and the same loss. We defer details, including architecture, to \citeauthor{grover2018stochastic}.

We compare to NeuralSort, which implements kNN using a relaxation of the sorting operator \cite{grover2018stochastic}.
We fix $k=9$ nearest neighbors to choose from $m$ candidates and search over temperatures $t=\{0.1, 1, 5, 16, 64\}$ using the validation set, whereas NeuralSort searches over both $k$ and $t$.
RelaxSubSample approaches the accuracy of a CNN trained using the standard cross entropy loss on MNIST (99.3\% vs. 99.4\%) and increases accuracy by 0.6\% over the NeuralSort implementation on CIFAR-10 (Table~\ref{fig:knntable}). 

Note that NeuralSort implements a relaxation to the sorting procedure, while in kNN we only require the top-$k$ elements. We use the top-$k$ relaxation from \cite{NN3}, which computes $k$ softmaxes for a runtime and storage of $O(km)$. NeuralSort requires $O(m^2)$ time and storage as it produces a $m \times m$ permutation matrix for each input. 
Table~\ref{fig:knntimings} shows forward pass runtimes of a CNN using our method and NeuralSort for different values of $m$ and $k=5$ neighbors. While runtimes for small $m$ are comparable, our method scales much better for larger $m$ (Table~\ref{fig:knntimings}).

\subsection{Stochastic Neighbor Embeddings}
\begin{table}[tb]
\centering
\resizebox{0.36\textwidth}{!} {
\begin{tabular}{lrrrr}  
\toprule
Model & d & MNIST & 20 Newsgroups \\
\midrule
Par. t-SNE ($\alpha=1$)  & 2  & 0.926 & 0.720   \\
RSS-SNE (no pretrain) & 2 & \textbf{0.929} & \textbf{0.763}\\
RSS-SNE (pretrain) & 2 &  \textbf{0.947} & \textbf{0.764} \\
\midrule
Par. t-SNE ($\alpha=1$)   & 10 & 0.983 & 0.854   \\
RSS-SNE (no pretrain) & 10 & \textbf{0.998} & \textbf{0.912}\\
RSS-SNE (pretrain) & 10 &\textbf{0.999} & \textbf{0.905}\\
\midrule
Par. t-SNE ($\alpha=1$)  & 30 & 0.983 & 0.866   \\
RSS-SNE (no pretrain) & 30 & \textbf{0.999} & \textbf{0.929}\\
RSS-SNE (pretrain) & 30 & \textbf{0.999} & \textbf{0.965}\\
\bottomrule
\end{tabular}
}
\caption{Trustworthiness(12) of low dimensional embeddings of size $d\in\{2,10,30\}$ for parametric t-SNE (Par. t-SNE) and RelaxSubSample SNE (RSS-SNE) on MNIST and 20 Newsgroups. Pretrain refers to layer-wise pretraining using autoencoders.
}
\label{fig:tsnetabletrustworthiness}
\end{table}

\begin{table}[tb]
\centering
\resizebox{0.36\textwidth}{!} {
\begin{tabular}{lrrrr}  
\toprule
Model    & d & MNIST & 20 Newsgroups \\
\midrule
Par. t-SNE ($\alpha=1$)  & 2  & 9.90 & 34.30   \\
RSS-SNE (no pretrain) & 2 & 11.80 & 36.80\\
RSS-SNE (pretrain) & 2 &  \textbf{8.31} & 35.11\\
\midrule
Par. t-SNE ($\alpha=1$)   & 10 & 5.38 & 24.40   \\
RSS-SNE (no pretrain) & 10 & \textbf{4.97} & 29.39\\
RSS-SNE (pretrain) & 10 &\textbf{4.56} & 28.50\\
\midrule
Par. t-SNE ($\alpha=1$)  & 30 & 5.41 & 24.88   \\
RSS-SNE (no pretrain) & 30 & \textbf{3.51} & 29.39\\
RSS-SNE (pretrain) & 30 & \textbf{3.05} & 28.90\\
\bottomrule
\end{tabular}
}
\caption{Test errors (\%) of 1-NN classifiers trained on low dimensional embeddings of size $d\in\{2,10,30\}$ generated by parametric t-SNE (Par. t-SNE) and our model (RSS-SNE). }
\label{fig:tsnetable1nn}
\end{table}

We consider the problem of learning a parametric embedding which maps a high-dimensional space into a lower-dimensional space while preserving local neighbor structure. This problem is addressed by parametric t-SNE \cite{pmlr-v5-maaten09a}, which represents pairwise densities of datapoints in the high-dimensional space as symmetrized Gaussians with variances tuned so that the perplexity of each point is equal. The pairwise densities in the low-dimensional space are modeled by Student-t distributions to address the \textit{crowding problem}, a result of the volume shrinkage from the high to low-dimensional spaces. Student-t distributions have heavier tails, allowing for distant points in the high-dimensional space to be modeled as far apart in the low-dimensional space. The objective is to minimize the KL divergence between the pairwise distributions in the two spaces.

\begin{figure}[tb]
  \centering
  \subfigure[MNIST]{
        \includegraphics[width=0.31\textwidth]{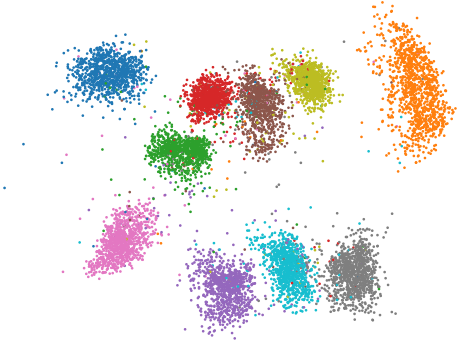}
    }
    \subfigure[20 Newsgroups]{
        \includegraphics[width=0.31\textwidth]{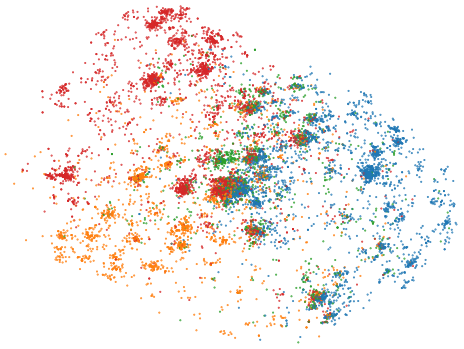}
    }
\caption{2D embeddings generated by our model (with pretrain) on 10000 MNIST datapoints and 15000 20 Newsgroups datapoints. Colors represent different classes. (Best in color.)}
\end{figure}

We propose to learn such a mapping without the use of Student-t distributions. Our insight is that the goal of preserving local structure can be achieved by preserving the neighbor \emph{rankings}, which is insensitive to scaling of distances. In our RelaxSubSample-based stochastic neighbor embedding (RSS-SNE), we aim to preserve the distribution of $k$ neighbors around each point. We define the neighbor distributions as follows. Let $X=\{x_1,\dots, x_n\}$ be the training data. Let $w(i,j)=\exp(-\|x_i-x_j\|_2^2)$ be exponentiated negative pairwise squared distances. For datapoint $x_i$, let $\mathbf{w}(i)\in \mathbb{R}^{n-1}$ be the pairwise distances from $x_i$ to other points. Then we model the neighbor distribution for $x_i$ as $p(S_{wrs} | \mathbf{w}(i))$ as in (\ref{subsetdist}). Note that we sample sub-sequences because we want to preserve neighbor rankings. Let $h$ be our parametric embedding function. Similarly, letting $\hat{w}(i,j)=\exp(-\|h(x_i)-h(x_j)\|_2^2)$ and the pairwise distances involving $h(x_i)$ be $\mathbf{\hat{w}}(i)$, the neighbor distribution in the low dimensional space is $p(S_{wrs} | \mathbf{\hat{w}}(i))$. We aim to match these distributions by comparing the neighbor \emph{samples} to avoid the crowding problem. 
For each $x_i$, let a sample from $p(S_{wrs} | \mathbf{w}(i))$ be $[\mathbf{e}^{i_1},\dots,\mathbf{e}^{i_k}]$ where $\mathbf{e}^{i_j}$ are 1-hot vectors corresponding to a selected neighbor $x_{i_j}$. 
Let a relaxed sample from $p(S_{wrs} | \mathbf{\hat{w}}(i))$ be $[\mathbf{a}^{i_1},\dots,\mathbf{a}^{i_k}]$, $k$ intermediate relaxed 1-hot outputs of the top-$k$ relaxation from \cite{NN3}. 
We minimize the objective
\begin{align}
    \frac{1}{n}\sum_{i=1}^n\sum_{j=1}^k \frac{1}{e^{j-1}} <\mathbf{e}^{i_j},-\log(\mathbf{a}^{i_j})>
\end{align}
where we aim to match samples from the neighbor distributions in the two spaces, putting more weight on matching closer neighbors.
While we can directly match the neighbor distributions $p(S_{wrs} | \mathbf{w(i)})$ and $p(S_{wrs} | \mathbf{\hat{w}}(i))$ without sampling, they are defined in terms of pairwise distances that cannot be matched due to the crowding problem. We find that sampling from both distributions is necessary to keep the loss agnostic to scaling of distances in the different spaces.

We compare with parametric t-SNE \cite{pmlr-v5-maaten09a} on the MNIST \cite{lecun-mnisthandwrittendigit-2010} and a small version of the 20 Newsgroups dataset \cite{20newsgroups}\footnote{Data at \protect\url{https://cs.nyu.edu/~roweis/data/20news_w100.mat}}.
Following \citeauthor{pmlr-v5-maaten09a}, we compare the trustworthiness and performance of 1-NN classifiers of the low dimensional embeddings on the \emph{test} set. Trustworthiness~\cite{Venna06visualizinggene} measures preservation of local structure in the low dimensional embedding and is defined as 
\[T(k)=1- \frac{2}{nk(2n-3k-1)}\sum_{i=1}^{n}\sum_{j\in N_i^{(k)}}\max(r(i,j)-k, 0) \]
where $r(i,j)$ is the rank of datapoint $j$ according to distances between datapoint $i$ and other datapoints in the high-dimensional space, and $N_i^{(k)}$ is the set of $k$ nearest neighbors in the low dimensional space. Trustworthiness decreases when a datapoint is in the $k$ nearest neighbors in the low dimensional space but not the original space. As in \citeauthor{pmlr-v5-maaten09a}, we use $T(12)$, comparing 12 nearest neighbors. 

We use the same feedforward networks as in \citeauthor{pmlr-v5-maaten09a} as embedding functions. For all experiments, we set $t=0.1$ and train for 200 epochs with a batch size of 1000, choosing the model with the best training loss. We sample neighbors only within each training batch.
We find that $k=1$ is sufficient to learn the local structure. For $k>1$, we observe a trade-off where 1-NN accuracy decreases but trustworthiness either stays the same or increases. It was important to add a small bias ($1\mathrm{e}-8$) to the relaxed $k$-hot vectors for better optimization. We compare two versions of RSS-SNE, one trained from scratch and another using layerwise pretraining. We pretrain layer $l$ by treating layers $1,\dots,l$ as an encoder and adding a 1 layer decoder to the original space; then we optimize the MSE autoencoder objective for 10 epochs. Note that the original parametric t-SNE used a similar layerwise pretraining scheme using RBMs. RSS-SNE models consistently have higher trustworthiness and competititive 1-NN test errors when compared to parametric t-SNE (Tables~\ref{fig:tsnetabletrustworthiness}, \ref{fig:tsnetable1nn}). Since trustworthiness compares 12 nearest neighbors, this suggests that our embedding has better overall structure as opposed to focusing on the immediate neighbor.

\section{Conclusion}
We present an algorithm for relaxing samples from a distribution over subsets such that the procedure can be included in deep models trained with backpropagation. We use the algorithm as a drop-in replacement in tasks requiring subset sampling to boost performance. Our algorithm has the potential to improve any task requiring subset sampling by tuning the model end-to-end with the subset procedure in mind.

\section*{Acknowledgements}
This research was supported by NSF (\#1651565, \#1733686), ONR, AFOSR (FA9550-
19-1-0024), and Bloomberg. We are thankful to Aditya Grover for helpful comments.

\bibliographystyle{named}
\bibliography{main}
\end{document}